\documentclass{article}
\usepackage{setspace}
\usepackage{amsmath}
\usepackage{amssymb}
\usepackage{titling}
\usepackage{comment}
\usepackage[usenames,dvipsnames]{xcolor}
\usepackage{fancybox}
\usepackage{fancyvrb}
\usepackage{amsthm}
\usepackage{graphicx}
\usepackage{mathrsfs}
\usepackage[footnotesize]{caption}
\usepackage{algorithm}
\usepackage[noend]{algpseudocode}
\usepackage{subcaption}
\usepackage{enumitem}
\setlist{nosep}

\usepackage[nonatbib,final]{neurips_2020}
\usepackage[utf8]{inputenc} 
\usepackage[T1]{fontenc}    
\usepackage{hyperref}       
\usepackage{url}            
\usepackage{booktabs}       
\usepackage{amsfonts}       
\usepackage{nicefrac}       
\usepackage{microtype}      
\usepackage[numbers]{natbib}

\DeclareMathOperator*{\argmin}{argmin}

\DeclareMathOperator{\prox}{prox}
\DeclareMathOperator{\dom}{dom}

\newtheorem{proposition}{Proposition}
\newtheorem{theorem}{Theorem}

\newcommand{\mon}{monDEQ}

\begin{document}
\setlength{\abovedisplayskip}{3.5pt}
\setlength{\belowdisplayskip}{3.5pt}
\setlength{\textfloatsep}{12pt}
\interfootnotelinepenalty=10000

\title{Monotone operator equilibrium networks}
\date{}
\author{Ezra Winston\\
School of Computer Science\\
Carnegie Mellon University\\
Pittsburgh, United States\\
\texttt{ewinston@cs.cmu.edu}
\And
J. Zico Kolter\\
School of Computer Science\\
Carnegie Mellon University\\
\& Bosch Center for AI\\
Pittsburgh, United States\\
\texttt{zkolter@cs.cmu.edu}
}

\maketitle

\begin{abstract}

Implicit-depth models such as Deep Equilibrium Networks have recently been shown to match or exceed the performance of traditional deep networks while being much more memory efficient. However, these models suffer from unstable convergence to a solution and lack guarantees that a solution exists. On the other hand, Neural ODEs, another class of implicit-depth models, do guarantee existence of a unique solution but perform poorly compared with traditional networks. In this paper, we develop a new class of implicit-depth model based on the theory of monotone operators, the Monotone Operator Equilibrium Network (\mon{}). We show the close connection between finding the equilibrium point of an implicit network and solving a form of monotone operator splitting problem, which admits efficient solvers with guaranteed, stable convergence. We then develop a parameterization of the network which ensures that all operators remain monotone, which guarantees the existence of a unique equilibrium point. Finally, we show how to instantiate several versions of these models, and implement the resulting iterative solvers, for structured linear operators such as multi-scale convolutions. The resulting models vastly outperform the Neural ODE-based models while also being more computationally efficient. Code is available at \url{http://github.com/locuslab/monotone_op_net}.

\end{abstract}

\section{Introduction}
Recent work in deep learning has demonstrated the power of \emph{implicit-depth} networks, models where features are created not by explicitly iterating some number of nonlinear layers, but by finding a solution to some implicitly defined equation.  Instances of such models include the Neural ODE \cite{chen2018neural}, which computes hidden layers as the solution to a continuous-time dynamical system, and the Deep Equilibrium (DEQ) Model \cite{bai2019deep}, which finds a fixed point of a nonlinear dynamical system corresponding to an effectively infinite-depth weight-tied network. These models, which trace back to some of the original work on recurrent backpropagation \cite{almeida1990learning,pineda1988generalization}, have recently regained attention since they have been shown to match or even exceed to performance of traditional deep networks in domains such as sequence modeling \cite{bai2019deep}. At the same time, these models show drastically improved memory efficiency over traditional networks since backpropagation is typically done analytically using the implicit function theorem, without needing to store the intermediate hidden layers.

However, implict-depth models that perform well require extensive tuning in order to achieve stable convergence to a solution. Obtaining convergence in DEQs requires careful initialization and regularization, which has proven difficult in practice \cite{linsley2020stable}. Moreover, solutions to these models are not guaranteed to exist or be unique, making the output of the models potentially ill-defined. While Neural ODEs \cite{chen2018neural} do guarantee existence of a unique solution, training remains unstable since the ODE problems can become severely ill-posed \cite{dupont2019augmented}. Augmented Neural ODEs \cite{dupont2019augmented} improve the stability of Neural ODEs by learning ODEs with simpler flows, but neither model achieves efficient convergence nor performs well on standard benchmarks. Crucial questions remain about how models can have guaranteed, unique solutions, and what algorithms are most efficient at finding them.

In this paper, we present a new class of implicit-depth equilibrium model, the Monotone Operator Equilibrium Network (\mon{}), which guarantees stable convergence to a unique fixed point.\footnote{We largely use the terms ``fixed point'' and ``equilibrium point'' interchangably in this work, using fixed point in the context of an iterative procedure, and equilibrium point to refer more broadly to the point itself.} The model is based upon the theory of monotone operators \cite{bauschke2011convex,ryu2016primer}, and illustrates a close connection between simple fixed-point iteration in weight-tied networks and the solution to a particular form of monotone operator splitting problem. Using this connection, this paper lays the theoretical and practical foundations for such networks. We show how to parameterize networks in a manner that ensures all operators remain monotone, which establishes the existence and uniqueness of the equilibrium point. We show how to backpropagate through such networks using the implicit function theorem; this leads to a corresponding (linear) operator splitting problem for the backward pass, which also is guaranteed to have a unique solution.  We then adapt traditional operator splitting methods, such as forward-backward splitting or Peaceman-Rachford splitting, to naturally derive algorithms for efficiently computing these equilibrium points.

Finally, we demonstrate how to practically implement such models and operator splitting methods, in the cases of typical feedforward, fully convolutional, and multi-scale convolutional networks.  For convolutional networks, the most efficient fixed-point solution methods require an inversion of the associated linear operator, and we illustrate how to achieve this using the fast Fourier transform. The resulting networks show strong performance on several benchmark tasks, vastly improving upon the accuracy and efficiency of Neural ODEs-based models, the other implicit-depth models where solutions are guaranteed to exist and be unique.

\vspace*{-.3\baselineskip}
\section{Related work}

\paragraph{Implicit models in deep learning}

There has been a growing interest in recent
years in implicit layers in deep learning. Instead of specifying the explicit computation to perform, a layer specifies some \emph{condition} that
should hold at the solution to the layer, such as a nonlinear equality, or a differential equation solution. Using the implicit function theorem allows for backpropagating through the layer solutions \emph{analytically}, making these layers very memory efficient, as they do not need to maintain intermediate iterations of the solution procedure. Recent examples include layers that compute inference in graphical models \cite{johnson2016composing}, solve optimization problems \cite{gould2016differentiating,amos2017optnet,gould2019deep,agarwal2019differentiable}, execute model-based control policies \cite{amos2018differentiable}, solve two-player games \cite{ling2018game}, solve gradient-based optimization for meta-learning \cite{rajeswaran2019meta}, and many others.

\vspace*{-.3\baselineskip}
\paragraph{Stability of fixed-point models}
The issue of model stability has in fact been at the heart of much work in fixed-point models. The original work on attractor-style recurrent models, trained via recurrent backpropagation \cite{almeida1990learning,pineda1988generalization}, precisely attempted to ensure that the forward iteration procedure was stable.  And indeed, much of the work in recurrent architectures such as LSTMs has focused on these issues of stability \cite{hochreiter1997long}. Recent work has revisited recurrent backpropagation in a similar manner to DEQs, with the similar aim of speeding up the computation of fixed points \cite{liao2018reviving}. And other work has looked at the stability of implicit models \cite{el2019implicit}, with an emphasis on guaranteeing the existence of fixed points, but focused on alternative stability conditions, and considered only relatively small-scale experiments. Other recent work has looked to use control-theoretic methods to ensure the stability of implicit models, \cite{revay2019contracting}, though again they consider only small-scale evaluations.


\vspace*{-.5\baselineskip}
\paragraph{Monotone operators in deep learning}

Although most work in the field of monotone operators is concerned with general convex analysis, the recent work of \cite{combettes2020deep} does also highlight connections between deep networks and monotone operator problems.  Unlike our current work, however, that work focused largely on the fact that many common non-linearities can be expressed via proximal operators, and analyzed traditional networks under the assumptions that certain of the operators were monotone, but did not address conditions for the networks to be monotone or algorithms for solving or backpropagating through the networks.

\vspace*{-.2\baselineskip}
\section{A monotone operator view of fixed-point networks}

This section lays out our main methodological and theoretical contribution, a class of equilibrium networks based upon monotone operators. We begin with some preliminaries, then highlight the basic connection between the fixed point of an ``infinite-depth'' network and an associated operator splitting problem; next, we propose a parameterization that guarantees the associated operators to be maximal monotone; finally, we show how to use operator splitting methods to both compute the fixed point and backpropagate through the fixed point efficiently.
\vspace*{-0.5\baselineskip}
\subsection{Preliminaries}

\paragraph{Monotone operator theory} The theory of monotone operators plays a foundational role in convex analysis and optimization. Monotone operators are a natural generalization of monotone functions, which can be used to assess the convergence properties of many forms of iterative fixed-point algorithms. We emphasize that the majority of the work in this paper relies on well-known properties about monotone operators, and we refer to standard references on the topic including \cite{bauschke2011convex} and a less formal survey by \cite{ryu2016primer}; we do include a brief recap of the definitions and results we require in Appendix \ref{apx:monotone}. Formally, an operator is a subset of the space $F \subseteq \mathbb{R}^n \times \mathbb{R}^n$; in our setting this will usually correspond to set-valued or single-valued function. Operator splitting approaches refer to methods for finding a zero in a sum of operators, i.e., find $x$ such that $0 \in (F + G)(x)$. There are many such methods, but the two we will use mainly in this work are \emph{forward-backward} splitting (eqn. \ref{eqn:forward-backward} in the Appendix) and \emph{Peaceman-Rachford} splitting (eqn. \ref{eqn:peaceman-rachford}). As we will see, both finding a network equilibrium point and backpropagating through it can be formulated as operator splitting problems, and different operator splitting methods will lead to different approaches in their application to our subsequent implicit networks.

\vspace*{-.3\baselineskip}
\paragraph{Deep equilibrium models} The \mon{} architecture is closely relate to the DEQ model, which parameterizes a ``weight-tied, input-injected'' network of the form $z_{i+1} = g(z_i, x)$, where $x$ denotes the input to the network, injected at each layer; $z_i$ denotes the hidden layer at depth $i$; and $g$ denotes a nonlinear function which is the same for each layer (hence the network is weight-tied). The key aspect of the DEQ model is that in this weight-tied setting,
instead of forward iteration, we can simply use any root-finding approach to find an equilibrium point of such an iteration $z^* = g(z^*, x)$. Assuming the model is stable, this equilibrium point corresponds to an ``infinite-depth fixed point'' of the layer. The \mon{} architecture can be viewed as an instance of a DEQ model, but one that relies on the theory of monotone operators, and a specific paramterization of the network weights, to guarantee the existence of a unique fixed point for the network. Crucially, however, as is the case for DEQs, naive forward iteration of this model is \emph{not} necessarily stable; we therefore employ operator splitting methods to develop provably (linearly) convergent methods for finding such fixed points.

\subsection{Fixed-point networks as operator splitting}

As a starting point of our analysis, consider the weight-tied, input-injected network in which $x \in \mathbb{R}^d$ denotes the input, and $z^k \in \mathbb{R}^n$ denotes the hidden units at layer $k$, given by the iteration\footnote{This setting can also be seen as corresponding to a recurrent network with identical inputs at each time (indeed, this is the view of so-called attractor networks \cite{pineda1988generalization}). However, because in modern usage recurrent networks typically refer to sequential models with \emph{different} inputs at each time, we don't adopt this terminology.}
\begin{equation}
\label{eqn:forward-iteration}
z^{k+1} = \sigma(W z^k + U x + b)
\end{equation}
where $\sigma : \mathbb{R} \rightarrow \mathbb{R}$ is a nonlinearity applied elementwise, $W \in \mathbb{R}^{n \times n}$ are the hidden unit weights, $U \in \mathbb{R}^{n \times x}$ are the input-injection weights and $b \in \mathbb{R}^n$ is a bias term.  An equilibrium point, or fixed point, of this system is some point $z^\star$ which remains constant after an update:
\begin{equation}
\label{forward-iteration}
z^\star = \sigma(W z^\star + U x + b).
\end{equation}
We begin by observing that it is possible to characterize this equilibrium point exactly as the solution to a certain operator splitting problem, under certain choices of operators and activation $\sigma$. This can be formalized in the following theorem, which we prove in Appendix \ref{apx:proofs}:

\begin{theorem}
\label{thm:monotone-formulation}
Finding a fixed point of the iteration (\ref{eqn:forward-iteration}) is equivalent to finding a zero of the operator splitting problem $0 \in (F+G)(z^\star)$ with the operators
\begin{equation}
F(z) = (I - W)(z) - (Ux + b), \;\; G = \partial f
\end{equation}
and $\sigma(\cdot) = \prox^1_f(\cdot)$ for some convex closed proper (CCP) function $f$, where  $\prox^\alpha_f$  denotes the proximal operator
\begin{equation}
\prox_f^\alpha(x) \equiv \argmin_z \frac{1}{2}\|x - z\|_2^2 + \alpha f(z).
\end{equation}
\end{theorem}
\vspace*{-.6\baselineskip}

It is also well-established that many common nonlinearities used in deep networks can be represented as proximal operators of CCP functions \cite{bibi2018deep,combettes2020deep}.  For example, the ReLU nonlinearity $\sigma(x) = [x]_+$ is the proximal operator of the indicator of the positive orthant $f(x) = I\{x \geq 0\}$, and tanh, sigmoid, and softplus all have close correspondence with proximal operators of simple expressions \cite{bibi2018deep}.

In fact, this method establishes that some seemingly unstable iterations can actually still lead to convergent algorithms. ReLU activations, for instance, have traditionally been avoided in iterative models such as recurrent networks, due to exploding or vanishing gradient problems and nonsmoothness.  Yet this iteration shows that (with input injection and the above constraint on $W$), ReLU operators are perfectly well-suited to these fixed-point iterations.

\vspace*{-.2\baselineskip}
\subsection{Enforcing existence of a unique solution}

The above connection is straightforward, but also carries interesting implications for deep learning.  Specifically, we can establish the existence and uniqueness of the equilibirum point $z^\star$ via the simple sufficient criterion that $I-W$ is strongly monotone, or in other words\footnote{For non-symmetric matrices, which of course is typically the case with $W$, positive definiteness is defined as the positive definiteness of the symmetric component $I - W \succeq mI \Leftrightarrow I - (W + W^T)/2 \succeq mI$.} $I - W \succeq mI$
for some $m > 0$ (see Appendix \ref{apx:monotone}). The constraint is by no means a trivial condition. Although many layers obey this condition under typical initialization schemes, during training it is normal for $W$ to move outside this regime.  Thus, the first step of the \mon{} architecture is to parameterize $W$ in such a way that it always satisfies this strong monotonicity constraint.

\begin{proposition}
We have $I - W \succeq mI$ if and only if there exist $A,B \in \mathbb{R}^{n \times n}$ such that
\begin{equation}
W = (1-m)I - A^T A + B - B^T.
\label{eqn:monotone-param}
\end{equation}
\label{prop:semidef}
\end{proposition}
\vspace*{-1\baselineskip}

We therefore propose to simply parameterize $W$ directly in this form, by defining the $A$ and $B$ matrices directly.  While this is an overparameterized form for a dense matrix, we could avoid this issue by, e.g. constraining $A$ to be lower triangular (making it the Cholesky factor of $A^T A$), and by making $B$ strictly upper triangular; in practice, however, simply using general $A$ and $B$ matrices has little impact upon the performance of the method.  The parameterization does notably raise additional complications when dealing with convolutional layers, but we defer this discussion to Section \ref{subsec:conv}.

\subsection{Computing the network fixed point}
\label{subsec:fixed_pt}
Given the \mon{} formulation, the first natural question to ask is: how should we compute the equilibrium point $z^\star = \sigma (Wz^\star + Ux +b)$?  Crucially, it can be the case that the simple forward iteration of the network equation (\ref{eqn:forward-iteration}) does \emph{not} converge, i.e., the iteration may be unstable. Fortunately, monotone operator splitting leads to a number of iterative methods for finding these fixed points, which are guaranteed to converge under proper conditions.  For example, the forward-backward iteration applied to the monotone operator formulation from Theorem \ref{thm:monotone-formulation} results exactly in a damped version of the forward iteration
\begin{equation}
z^{k+1} = \prox_{f}^{\alpha}(z^k - \alpha((I-W)z^k - (U x + b))) = \prox_{f}^{\alpha}((1-\alpha)z^k + \alpha(Wz^k + Ux + b)).
\end{equation}
This iteration is guaranteed to converge linearly to the fixed point $z^\star$ provided that $\alpha \leq 2m/L^2$, when the operator $I-W$ is Lipschitz and strongly monotone with parameters $L$ (which is simply the operator norm $\|I-W\|_2$) and $m$ \cite{ryu2016primer}.

A key advantage of the \mon{} formulation is the flexibility to employ alternative operator splitting methods that converge much more quickly to the equilibrium.  One such example is Peaceman-Rachford splitting which, when applied to the formulation from Theorem \ref{thm:monotone-formulation}, takes the form
\begin{equation}
\begin{split}
u^{k+1/2} & = 2 z^k - u^k \\
z^{k+1/2} & = (I + \alpha(I-W))^{-1}(u^{k+1/2} - \alpha (Ux +b)) \\
u^{k+1} & = 2z^{k+1/2} - u^{k+1/2} \\
z^{k+1} & = \prox_f^\alpha (u^{k+1}) \\
\end{split}
\end{equation}
where we use the explicit form of the resolvents for the two monotone operators of the model. The advantage of Peaceman-Rachford splitting over forward-backward is two-fold: 1) it typically converges in fewer iterations, which is a key bottleneck for many implicit models; and 2) it converges for any $\alpha > 0$ \cite{ryu2016primer}, unlike forward-backward splitting which is dependent on the Lipschitz constant of $I-W$.  The disadvantage of Peaceman-Rachford splitting, however, is that it requires an inverse involving the weight matrix $W$. It is not immediately clear how to apply such methods if the $W$ matrix involves convolutions or multi-layer models; we discuss these points in Section \ref{subsec:conv}. A summary of these methods for computation of the forward equilibrium point is given in Algorithms \ref{alg:FB-forward} and \ref{alg:PR-forward}.

\begin{figure*}[t]
\begin{minipage}[t]{.49\textwidth}
    \begin{algorithm}[H]
    \footnotesize
    \caption{Forward-backward equilibrium solving}
    \label{alg:FB-forward}
    \begin{algorithmic}[0]
    \State $z := 0$;~~~~~$\text{err} := 1$
    \While{$\text{err} > \epsilon$}
        \State $z^+ := (1-\alpha)z + \alpha (Wz + Ux + b)$
        \State $z^+ := \prox_{f}^{\alpha}(z^+)$
        \State $\text{err} := \frac{\|z^+-z\|_2}{\|z^+\|_2}$
        \State $z := z^+$
    \EndWhile
    \State \textbf{return} $z$
    \end{algorithmic}
    \end{algorithm}
\end{minipage}
\hfill
\begin{minipage}[t]{.49\textwidth}
    \begin{algorithm}[H]
    \footnotesize
    \caption{Peaceman-Rachford equilibrium solving}
    \label{alg:PR-forward}
    \begin{algorithmic}[0]
    \State $z, u := 0$;~~~~~$\text{err} := 1$;~~~~~$V := (I + \alpha(I-W))^{-1}$
    \While{$\text{err} > \epsilon$}
        \State $u^{1/2} := 2z - u$
        \State $z^{1/2} := V(u^{1/2} + \alpha(Ux + b))$
        \State $u^+ := 2z^{1/2} - u^{1/2}$
        \State $z^+ := \prox_{f}^{\alpha}(u^+)$
        \State $\text{err} := \frac{\|z^+-z\|_2}{\|z^+\|_2}$
        \State $z,u := z^+, u^+$
    \EndWhile
    \State \textbf{return} $z$
    \end{algorithmic}
    \end{algorithm}
\end{minipage}
\end{figure*}

\subsection{Backprogation through the monotone operator layer}
\label{subsec:backprop}
Finally, a key challenge for any implicit model is to determine how to perform backpropagation through the layer.  As with most implicit models, a potential benefit of the fixed-point conditions we describe is that, by using the implicit function theorem, it is possible to perform backpropagation without storing the intermediate iterates of the operator splitting algorithm in memory, and instead backpropagating directly through the equilibrium point.

To begin, we present a standard approach to differentiating through the fixed point $z^\star$ using the implicit function theorem. This formulation has some compelling properties for \mon{}, namely the fact that this (sub)gradient will always exist.  When training a network via gradient descent, we need to compute the gradients of the loss function
\begin{equation}
\frac{\partial \ell}{\partial (\cdot) } = \frac{\partial \ell}{\partial z^\star} \frac{\partial z^\star}{\partial (\cdot)}
\end{equation}
where $(\cdot)$ denotes some input to the layer or parameters, i.e. $W$, $x$, etc.  The challenge here is computing (or left-multiplying by) the Jacobian $\partial z^\star / \partial (\cdot)$, since $z^\star$ is not an explicit function of the inputs.  While it would be possible to simply compute gradients through the ``unrolled'' updates, e.g. $z^{k+1} = \sigma(W z^k + Ux + b)$ for forward iteration, this would require storing each intermediate state $z^{k}$, a potentially memory-intensive operation.  Instead, the following theorem gives an explicit formula for the necessary (sub)gradients. We state the theorem more directly in terms of the operators mentioned Theorem \ref{thm:monotone-formulation}; that is, we use $\prox_f^1(\cdot)$ in place of $\sigma(\cdot)$.

\begin{theorem}
\label{thm:backward}
For the equilibrium point $z^\star = \prox_f^1(Wz^\star + Ux + b)$, we have
\begin{equation}
\frac{\partial \ell}{\partial (\cdot) } =
\frac{\partial \ell}{\partial z^\star} (I - J W)^{-1} J \frac{\partial (Wz^\star + U x + b)}{\partial (\cdot)}
\end{equation}
where
\begin{equation}
\label{eqn:J}
J = \mathsf{D} \prox_f^1 (W z^\star + U x + b)
\end{equation}
denotes the Clarke generalized Jacobian of the nonlinearity evaluated at the point $W z^\star + U x + b$.  Furthermore, for the case that $(I - W) \succeq mI$, this derivative always exists.
\end{theorem}

To apply the theorem in practice to perform reverse-mode differentiation, we need to solve the system
\begin{equation}
(I - J W)^{-T} \left(\frac{\partial \ell}{\partial z^\star}\right )^T.
\label{eqn:gradient}
\end{equation}
The above system is a linear equation and while it is typically computationally infeasible to compute the inverse $(I - J W)^{-T}$ exactly, we could compute a solution to $(I - J W)^{-T}v$ using, e.g., conjugate gradient methods. However, we present an alternative formulation to computing \eqref{eqn:gradient} as the solution to a (linear) monotone operator splitting problem:

\vspace*{.4\baselineskip}
\begin{theorem}
\label{thm:backward-monotone}
Let $z^\star$ be a solution to the monotone operator splitting problem defined in Theorem \ref{thm:monotone-formulation}, and define $J$ as in
(\ref{eqn:J}). Then for $v \in \mathbb{R}^n$ the solution of the equation
\begin{equation}
u^\star = (I - J W)^{-T}v
\end{equation}
is given by
\begin{equation}
u^\star = v + W^T \tilde{u}^\star
\end{equation}
where $\tilde{u}^\star$ is a zero of the operator splitting problem $0 \in (\tilde{F}+\tilde{G})(u^\star)$, with operators defined as
\begin{equation}
\tilde{F}(\tilde{u}) = (I - W^T)(\tilde{u}) , \;\; \tilde{G}(\tilde{u}) = D \tilde{u} - v
\label{eqn:splitting_bw}
\end{equation}
where $D$ is a diagonal matrix defined by $J = (I+D)^{-1}$ (where we allow for the possibility of $D_{ii} = \infty$ for $J_{ii} = 0$).
\end{theorem}

An advantage of this approach when using Peaceman-Rachford splitting is that it allows us to reuse a fast method for multiplying by $(I + \alpha(I-W))^{-1}$ which is required by Peaceman-Rachford during both the forward pass (equilibrium solving) and backward pass (backpropagation) of training a \mon{}. Algorithms detailing both the Peaceman-Rachford and forward-backward solvers for the backpropagation problem \eqref{eqn:splitting_bw} are given in Algorithms \ref{alg:FB-backward} and \ref{alg:PR-backward}.

\begin{figure*}[t]
\begin{minipage}[t]{.4\textwidth}
    \begin{algorithm}[H]
    \small
    \caption{Forward-backward equilibrium backpropagation}
    \label{alg:FB-backward}
    \begin{algorithmic}[0]
    \State $u := 0$;~~~~~$\text{err} := 1$;~~~~~$v := \frac{\partial \ell}{\partial z^*}$
    \While{$\text{err} > \epsilon$}
        \State $u^+ := (1-\alpha)u + \alpha W^T u$
        \State $u^+_i :=
            \begin{cases} \frac{u^+_i + \alpha v_i}{1 + \alpha(1+D_{ii})} &\mbox{if } D_{ii} < \infty \\
                            0 & \mbox{if } D_{ii} = \infty
            \end{cases}$
        \State $\text{err} := \frac{\|u^+-u\|_2}{\|u^+\|_2}$
        \State $u := u^+$
    \EndWhile
    \State \textbf{return} $u$
    \end{algorithmic}
    \end{algorithm}
\end{minipage}
\hfill
\begin{minipage}[t]{.58\textwidth}
    \begin{algorithm}[H]
    \small
    \caption{Peaceman-Rachford equilibrium backpropagation}
    \label{alg:PR-backward}
    \begin{algorithmic}[0]
    \State $z, u := 0$;~~~~~$\text{err} := 1$;~~~~~$v := \frac{\partial \ell}{\partial z^*}$;~~~~~$V := (I + \alpha(I-W))^{-1}$
    \While{$\text{err} > \epsilon$}
        \State $u^{1/2} := 2z - u$
        \State $z^{1/2} := V^T u^{1/2}$
        \State $u^+ := 2z^{1/2} - u^{1/2}$
        \State $z^+_i :=
            \begin{cases} \frac{u^+_i + \alpha v_i}{1 + \alpha(1+D_{ii})} &\mbox{if } D_{ii} < \infty \\
                            0 & \mbox{if } D_{ii} = \infty
            \end{cases}$
        \State $\text{err} := \frac{\|z^+-z\|_2}{\|z^+\|_2}$
        \State $z,u := z^+, u^+$
    \EndWhile
    \State \textbf{return} $z$
    \end{algorithmic}
    \end{algorithm}
\end{minipage}
\end{figure*}
\vspace*{-.3\baselineskip}
\section{Example monotone operator networks}
\label{sec:example_mons}
With the basic foundations from the previous section, we now highlight several different instantiations of the \mon{} architecture. In each of these settings, as in Theorem \ref{thm:monotone-formulation}, we will formulate the objective as one of finding a solution to the operator splitting problem $0 \in (F+G)(z^\star)$ for
\begin{equation}
F(z) = (I - W)(z) - (Ux + b), \;\; G = \partial f
\end{equation}
or equivalently as computing an equilibrium point $z^\star = \prox_f^1(Wz^\star + U x + b)$.

In each of these settings we need to define what the input and hidden state $x$ and $z$ correspond to, what the $W$ and $U$ operators consist of, and what is the function $f$ which determines the network nonlinearity. Key to the application of monotone operator methods are that 1) we need to constrain the $W$ matrix such that $I - W \succeq mI$ as described in the previous section and 2) we need a method to compute (or solve) the inverse $(I + \alpha(I - W))^{-1}$, needed e.g. for Peaceman-Rachford; while this would not be needed if using only forward-backward splitting, we believe that the full power of the monotone operator view is realized precisely when these more involved methods are possible.
\vspace*{-.2\baselineskip}
\subsection{Fully connected networks}

The simplest setting, of course, is the case we have largely highlighted above, where $x \in \mathbb{R}^d$ and $z \in \mathbb{R}^n$ are unstructured vectors, and $W \in \mathbb{R}^{n \times n}$ and $U \in \mathbb{R}^{n \times d}$ and $b \in \mathbb{R}^n$ are dense matrices and vectors respectively.  As indicated above, we parameterize $W$ directly by $A,B \in \mathbb{R}^{n \times n}$ as in \eqref{eqn:monotone-param}.
Since the $Ux$ term simply acts as a bias in the iteration, there is no constraint on the form of $U$.

We can form an inverse directly by simply forming and inverting the matrix $I + \alpha(I-W)$, which has cost $O(n^3)$.  Note that this inverse needs to be formed only once, and can be reused over all iterations of the operator splitting method and over an entire batch of examples (but recomputed, of course, when $W$ changes).  Any proximal function can be used as the activation: for example the ReLU, though as mentioned there are also close approximations to the sigmoid, tanh, and softplus.
\vspace*{-.2\baselineskip}
\subsection{Convolutional networks}
\label{subsec:conv}
The real power of \mon{}s comes with the ability to use more structured linear operators such as convolutions.  We let $x \in \mathbb{R}^{ds^2}$ be a $d$-channel input of size $s \times s$ and $z \in \mathbb{R}^{ns^2}$ be a $n$-channel hidden layer. We also let $W \in \mathbb{R}^{ns^2 \times ns^2}$ denote the linear form of a 2D convolutional operator and similarly for $U \in \mathbb{R}^{ns^2 \times ds^2}$.  As above, $W$ is parameterized by two additional convolutional operators $A,B$ of the same form as $W$.  Note that this implicitly increases the receptive field size of $W$: if $A$ and $B$ are $3\times 3$ convolutions, then $W = (1-m)I - A^T A + B - B^T$ will have an effective kernel size of 5.

\paragraph{Inversion} The benefit of convolutional operators in this setting is the ability to perform efficient inversion via the fast Fourier transform. Specifically, in the case that $A$ and $B$ represent circular convolutions, we can reduce the matrices to block-diagonal form via the discrete Fourier transform (DFT) matrix
\begin{equation}
A = F_s D_A F_s^*
\end{equation}
where $F_s$ denotes (a permuted form of) the 2D DFT operator and $D_A \in \mathbb{C}^{ns^2 \times n s^2}$ is a (complex) block diagonal matrix where each block ${D_A}_{ii} \in \mathbb{C}^{n \times n}$ corresponds to the DFT at one particular location in the image. In this form, we can efficiently multiply by the inverse of the convolutional operator, noting that
\begin{equation}
\label{FFT-diagonalization}
\begin{split}
I + \alpha(I-W) & = (1+\alpha m)I + \alpha A^T A - \alpha B + \alpha B^T \\
& = F_s ((1+\alpha m)I + \alpha D_A^* D_A - D_B + D_B^*) F_s^*.
\end{split}
\end{equation}
The inner term here is itself a block diagonal matrix with complex $n \times n$ blocks (each block is also guaranteed to be invertible by the same logic as for the full matrix).  Thus, we can multiply a set of hidden units $z$ by the inverse of this matrix by simply inverting each $n \times n$ block, taking the fast Fourier transform (FFT) of $z$, multiplying each corresponding block of $F_s z$ by the corresponding inverse, then taking the inverse FFT. The details are given in Appendix \ref{apx:conv-mons}.

The computational cost of multiplying by this inverse is $O(n^2 s^2 \log s + n^3 s^2)$ to compute the FFT of each convolutional filter and precompute the inverses, and then $O(b n s^2 \log s + b n^2 s^2)$ to multiply by the inverses for a set of hidden units with a minibatch of size $b$.  Note that just computing the convolutions in a normal manner has cost $O(b n^2 s^2)$, so that these computations are on the same order as performing typical forward passes through a network, though empirically 2-3 times slower owing to the relative complexity of performing the necessary FFTs.

One drawback of using the FFT in this manner is that it requires that all convolutions be circular; however, this circular dependence can be avoided using zero-padding, as detailed in Section \ref{apx:zero-padding}.

\vspace*{-.2\baselineskip}
\subsection{Forward multi-tier networks}
\label{subsec:multi-tier}
Although a single fully-connected or convolutional operator within a \mon{} can be suitable for small-scale problems, in typical deep learning settings it is common to model hidden units at different hierarchical levels.  While \mon{}s may seem inherently ``single-layer,'' we can model this same hierarchical structure by incorporating structure into the $W$ matrix.  For example, assuming a convolutional setting, with input $x \in \mathbb{R}^{ds^2}$ as in the previous section, we could partition $z$ into $L$ different hierarchical resolutions and let $W$ have a multi-tier structure, e.g.
\begin{equation*}
z = \left [ \begin{array}{c} z_1 \in \mathbb{R}^{n_1 s_1^2} \\
z_2 \in \mathbb{R}^{n_2 s_2^2} \\
\vdots \\
z_L \in \mathbb{R}^{n_L s_L^2}
\end{array} \right ],
~~~~~~~W = \left [ \begin{array}{ccccc}
W_{11} & 0 & 0 & \cdots & 0  \\
W_{21} & W_{22} & 0 & \cdots & 0 \\
0 & W_{32} & W_{33} & \cdots & 0  \\
\vdots & \vdots & \vdots & \ddots & \vdots \\
0 & 0 & 0 & \cdots & W_{LL} \end{array} \right ]
\end{equation*}
where $z_i$ denotes the hidden units at level $i$, an $s_i \times s_i$ resolution hidden unit with $n_i$ channels, and where $W_{ii}$ denotes an $n_i$ channel to $n_i$ channel convolution, and $W_{i+1,i}$ denotes an $n_i$ to $n_{i+1}$ channel, \emph{strided} convolution. This structure of $W$ allows for both inter- and intra-tier influence.

\begin{figure}[!t]
\begin{minipage}{.48\textwidth}
	\centering
	\small
	\makeatletter
	\setlength{\tabcolsep}{4pt}
	\begin{tabular}{l c r}
		\multicolumn{3}{c}{\textbf{CIFAR-10}} \\ \midrule
		\textbf{Method}      & \textbf{Model size} & \textbf{Acc.}\\  \midrule
		Neural ODE& 172K & 55.3$\pm$0.3\%         \\
		Aug. Neural ODE& 172K & 58.9$\pm$2.8\%      \\
		Neural ODE$^{\dagger*}$& 1M & 59.9\%         \\
		Aug. Neural ODE$^{\dagger*}$& 1M & 73.4\%      \\
		\midrule
		\mon{} (ours) & & \\\midrule
		\hspace{8pt} Single conv & 172K & \textbf{74.0$\pm$0.1\%}              \\
		\hspace{8pt} Multi-tier & 170K & 72.0$\pm$0.3\%        \\
		\hspace{8pt} Single conv$^*$& 854K & 82.0$\pm$0.3\%  \\
		\hspace{8pt} Multi-tier$^*$& 1M & \textbf{89.0$\pm$0.3\%}  \\  
		[0.05in] \multicolumn{3}{c}{\textbf{SVHN}}\\ \midrule
		\textbf{Method}      & \textbf{Model size} & \textbf{Acc.}\\  \midrule
		Neural ODE$^{\ddagger}$& 172K & 81.0\% \\
		Aug. Neural ODE$^{\ddagger}$& 172K & 83.5\%      \\\midrule
		\mon{} (ours) & & \\\midrule
		\hspace{8pt} Single conv & 172K & 88.7$\pm$1.1\%            \\
		\hspace{8pt} Multi-tier & 170K & \textbf{92.4$\pm$0.1\%}          \\ [0.05in]
		\multicolumn{3}{c}{\textbf{MNIST}}\\ \midrule
		\textbf{Method}      & \textbf{Model size} & \textbf{Acc.} \\ \midrule
		Neural ODE$^{\ddagger}$&  84K & 96.4\%         \\
		Aug. Neural ODE$^{\ddagger}$&  84K & 98.2\%      \\\midrule
		\mon{} (ours) & & \\\midrule
		\hspace{8pt} Fully connected & 84K & 98.1$\pm$0.1\%            \\
		\hspace{8pt} Single conv & 84K & \textbf{99.1$\pm$0.1\%}             \\
		\hspace{8pt} Multi-tier & 81K & 99.0$\pm$0.1\%          \\
        \\
	\end{tabular}
	\renewcommand{\@captype}{table}
    \caption{Test accuracy of \mon{} models compared to Neural ODEs and Augmented Neural ODEs. *with data augmentation; $^\dagger$best test accuracy before training diverges; $^\ddagger$as reported in \cite{dupont2019augmented}.}
	\label{tab:results}
	\makeatother
\end{minipage}
\hfill
\begin{minipage}{.50\textwidth}
    \centering
    \includegraphics{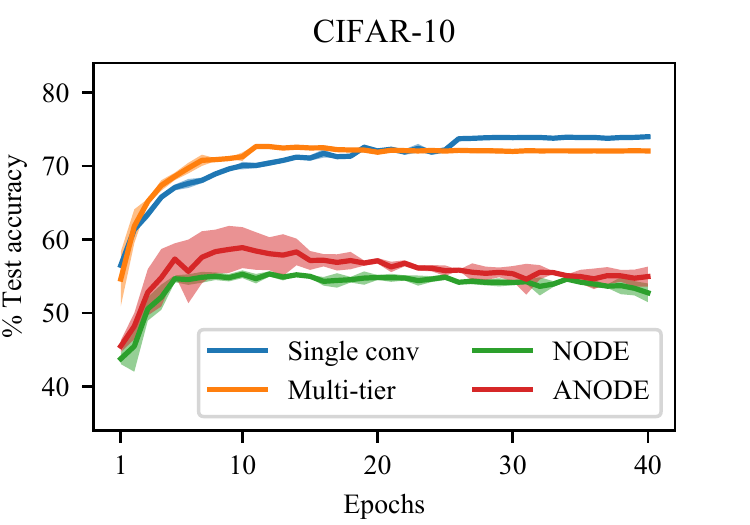}
    \caption{Test accuracy of \mon{}s during training on CIFAR-10, with NODE \cite{chen2018neural} and ANODE \cite{dupont2019augmented} for comparison. NODE and ANODE curves obtained using code provided by \cite{dupont2019augmented}.}
    \label{fig:cifar_train}
    \centering
    \includegraphics{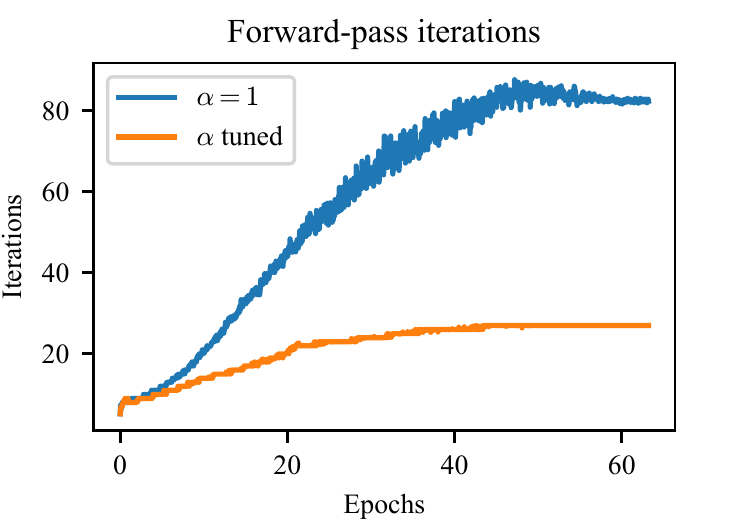}
    \caption{Iterations required by Peaceman-Rachford equilibrium solving over the course of training, for best $\alpha$ and $\alpha=1$.}
    \label{fig:tuning_alpha_fw}
\end{minipage}
\end{figure}
\begin{figure}[t]
\centering
\includegraphics[width=\textwidth]{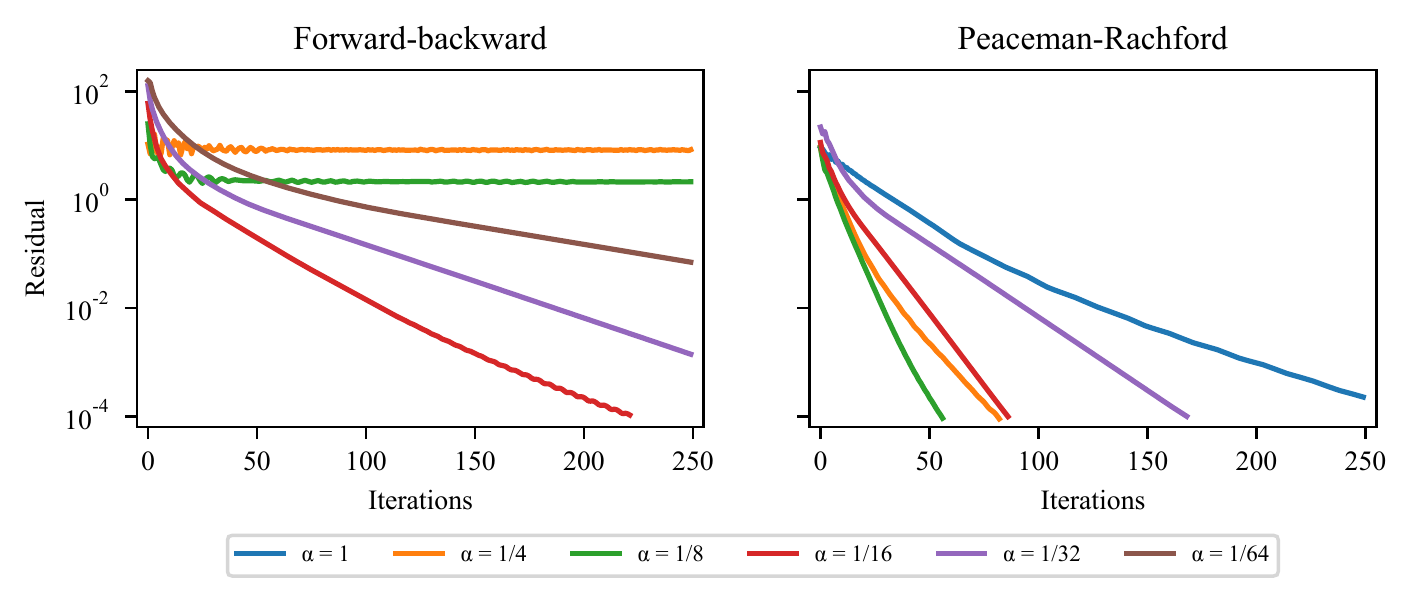}
\caption{Convergence of Peaceman-Rachford and forward-backward equilibrium solving, on fully-trained model.}
\label{fig:convergence}
\end{figure}

One challenge is to ensure that we can represent $W$ with the form $(1-m)I - A^T A + B - B^T$ while still maintaining the above structure, which we achieve by parameterizing each $W_{ij}$ block appropriately. Another consideration is the inversion of the multi-tier operator, which can be achieved via the FFT similarly as for single-convolutional $W$, but with additional complexity arising from the fact that the $A_{i+1,i}$ convolutions are strided. These details are described in Appendix \ref{apx:multi-tier-inverse}.

\section{Experiments}
To test the expressive power and training stability of \mon{}s, we evaluate the \mon{} instantiations described in Section \ref{sec:example_mons} on several image classification benchmarks. We take as a point of comparison the Neural ODE (NODE) \citep{chen2018neural} and Augmented Neural ODE (ANODE) \citep{dupont2019augmented} models, the only other implicit-depth models which guarantee the existence and uniqueness of a solution. We also assess the stability of training standard DEQs of the same form as our \mon{}s.

The training process relies upon the operator splitting algorithms derived in Sections \ref{subsec:fixed_pt} and \ref{subsec:backprop}; for each batch of examples, the forward pass of the network involves finding the network fixed point (Algorithm \ref{alg:FB-forward} or \ref{alg:PR-forward}), and the backward pass involves backpropagating the loss gradient through the fixed point (Algorithm \ref{alg:FB-backward} or \ref{alg:PR-backward}). We analyze the convergence properties of both the forward-backward and Peaceman-Rachford operator splitting methods, and use the more efficient Peaceman-Rachford splitting for our model training. For further training details and model architectures see Appendix \ref{apx:experiment_details}.  Experiment code can be found at \url{http://github.com/locuslab/monotone_op_net}.
\paragraph{Performance on image benchmarks}
\label{subsec:results}
We train small \mon{}s on CIFAR-10 \citep{krizhevsky2009learning}, SVHN \citep{netzer2011reading}, and MNIST \citep{lecun2010mnist}, with a similar number of parameters as the ODE-based models reported in \cite{chen2018neural} and \cite{dupont2019augmented}. The results (averages over three runs) are shown in Table \ref{tab:results}. Training curves for \mon{}s, NODE, and ANODE on CIFAR-10 are show in Figure (\ref{fig:cifar_train}) and additional training curves are shown in Figure \ref{fig:training}. Notably, except for the fully-connected model on MNIST, all \mon{}s significantly outperform the ODE-based models across datasets. We highlight the performance of the small single convolution \mon{} on CIFAR-10 which outperforms Augmented Neural ODE by 15.1\%. 

We also attempt to train standard DEQs of the same structure as our small multi-tier convolutional \mon{}. We train DEQs both with unconstrained $W$ and with $W$ having the monotone parameterization (\ref{eqn:monotone-param}), and solve for the fixed point using Broyden's method as in \cite{bai2019deep}. All models quickly diverge during the first few epochs of training, even when allowed 300 iterations of Broyden's method.

Additionally, we train two larger \mon{}s on CIFAR-10 with data augmentation. The strong performance (89\% test accuracy) of the multi-tier network, in particular, goes a long way towards closing the performance gap with traditional deep networks. For comparison, we train larger NODE and ANODE models with a comparable number of parameters (\textasciitilde1M). These attain higher test accuracy than the smaller models during training, but diverge after 10-30 epochs (see Figure \ref{fig:training}). 

\paragraph{Efficiency of operator splitting methods}
\label{subsec:convergence}
We compare the convergence rates of Peaceman-Rachford and forward-backward splitting on a fully trained model, using a large multi-tier \mon{} trained on CIFAR-10. Figure \ref{fig:convergence} shows convergence for both methods during the forward pass, for a range of $\alpha$. As the theory suggests, the convergence rates depend strongly on the choice of $\alpha$. Forward-backward does not converge for $\alpha > 0.125$, but convergence speed varies inversely with $\alpha$ for $\alpha < 0.125$. In contrast, Peaceman-Rachford is guaranteed to converge for any $\alpha > 0$ but the dependence is non-monotonic. We see that, for the optimal choice of $\alpha$, Peaceman-Rachford can converge much more quickly than forward-backward. The convergence rate also depends on the Lipschitz parameter $L$ of $I-W$, which we observe increases during training. Peaceman-Rachford therefore requires an increasing number of iterations during both the forward pass (Figure \ref{fig:tuning_alpha_fw}) and backward pass (Figure \ref{fig:tuning_alpha_bw}).

Finally, we compare the efficiency of \mon{} to that of the ODE-based models. We report the time and number of function evaluations (OED solver steps or operator splitting iterations) required by the \textasciitilde 170k-parameter models to train on CIFAR-10 for 40 epochs. The \mon{}, neural ODE, and ANODE training takes respectively 1.4, 4.4, and 3.3 hours, with an average of 20, 96, and 90 function evals per minibatch. Note however that training the larger 1M-parameter \mon{} on CIFAR-10 requires 65 epochs and takes 16 hours. All experiments are run on a single RTX 2080 Ti GPU.

\section{Conclusion}
The connection between monotone operator splitting and implicit network equilibria brings a new suite of tools to the study of implicit-depth networks. The strong performance, efficiency, and guaranteed stability of \mon{} indicate that such networks could become practical alternatives to deep networks, while the flexibility of the framework means that performance can likely be further improved by, e.g. imposing additional structure on $W$ or employing other operator splitting methods. At the same time, we see potential for the study of \mon{}s to inform traditional deep learning itself. The guarantees we can derive about what architectures and algorithms work for implicit-depth networks may give us insights into what will work for explicit deep networks.

\section*{Broader impact statement}
While the main thrust of our work is foundational in nature, we do demonstrate the potential for implicit models to become practical alternatives to traditional deep networks. Owing to their improved memory efficiency, these networks have the potential to further applications of AI methods on edge devices, where they are currently largely impractical. However, the work is still largely algorithmic in nature, and thus it is much less clear the immediate societal-level benefits (or harms) that could result from the specific tehniques we propose and demonstrate in this paper.

\section*{Acknowledgements}
Ezra Winston is supported by a grant from the Bosch Center for Artificial Intelligence.

\bibliographystyle{abbrvnat}
\bibliography{deep_mo}

\newpage
\appendix

\section{Monotone operator theory}
\label{apx:monotone}
\renewcommand{\thetable}{\Alph{section}\arabic{table}}
\setcounter{table}{0}
\renewcommand{\thefigure}{\Alph{section}\arabic{figure}}
\setcounter{figure}{0}
\renewcommand{\thealgorithm}{\Alph{section}\arabic{algorithm}}
\setcounter{algorithm}{0}
\renewcommand{\theequation}{\Alph{section}\arabic{equation}}
\setcounter{equation}{0}

We briefly review some of the basic properties of monotone operators that we make use of throughout this work. A \emph{relation} or \emph{operator} (which in our setting will often roughly correspond to a set-valued function), is a subset of the space $F \subseteq \mathbb{R}^n \times \mathbb{R}^n$; we use the notation $F(x) = \{y | (x,y) \in F\}$ or simply $F(x) = y$ if only a single $y$ is contained in this set.  We make use of a few basic operators and relations: the identity operator $I = \{(x,x) |  x \in \mathbb{R}^n\}$; the operator sum $(F + G)(x) = \{(x, y + z) | (x,y) \in F, (x,z) \in G\}$; the inverse operator $F^{-1}(x,y) = \{(y,x) | (x,y) \in F\}$; and the subdifferential operator $\partial f = \{ (x, \partial f(x)) | \ \in \mathrm{dom} f\}$.  An operator $F$ has Lipschitz constant $L$ if for any $(x,u), (y,v) \in F$
\begin{equation}
\|u-v\|_2 \leq L \|x - y\|_2.
\end{equation}

An operator $F$ is monotone if
\begin{equation}
(u - v)^T(x-y) \geq 0, \;\; \forall (x,u), (y,v) \in F
\end{equation}
which for the case of $F$ being a function $F : \mathbb{R}^n \rightarrow \mathbb{R}^n$ is equivalent to the condition
\begin{equation}
(F(x) - F(y))^T(x-y) \geq 0, \;\; \forall x,y \in \dom F.
\end{equation}
In the case of scalar-valued functions, this corresponds to our common notion of a monotonic function.  The operator $F$ is strongly monotone with parameter $m$ if
\begin{equation}
(u - v)^T(x-y) \geq m \|x - y\|^2, \;\; \forall (x,u), (y,v) \in F.
\end{equation}
A monotone operator $F$ is \emph{maximal monotone} if no other monotone operator strictly contains it; formally, most of the convergence properties we use require maximal monotonicity, though we are intentionally informal about this and merely use the established fact that several well-known operators are maximal monotone.  Specifically, a linear operator $F(x) = G x + h$ for $G \in \mathbb{R}^{n \times n}$ and $h \in \mathbb{R}^n$ is (maximal) monotone if and only if $G + G^T \succeq 0$ and strongly monotone if $G + G^T \succeq m I$.  Similarly, a subdifferentiable operator $\partial f$ is maximal monotone iff $f$ is a convex closed proper (CCP) function.

The resolvent and Cayley operators for an operator $F$ are denoted $R_F$ and $C_F$ and respectively defined as
\begin{equation}
R_F = (I + \alpha F)^{-1}, \;\; C_F = 2 R_F - I
\end{equation}
for any $\alpha > 0$.  The resolvent and Cayley operators are non-expansive (i.e., have Lipschitz constant $L \leq 1$) for any maximal monotone $F$, and are contractive (i.e. $L < 1$) for strongly monotone $F$.

We will mainly use two well-known properties of these operators.  First, when $F(x) = G x + h$ is linear, then
\begin{equation}
R_F(x) = (I + \alpha G)^{-1}(x - \alpha h)
\end{equation}
and when $F = \partial f$ for some CCP function $f$, then the resolvent is given by a proximal operator
\begin{equation}
R_F(x) = \prox_f^\alpha(x) \equiv \argmin_z \frac{1}{2}\|x - z\|_2^2 + \alpha f(z).
\end{equation}

Operator splitting approaches refer to methods to find a zero in a sum of operators (assumed here to be maximal monotone), i.e., find $x$ such that
\begin{equation}
0 \in (F + G)(x).
\end{equation}
There are many such operator splitting methods, which lead to different approaches in their application to our subsequent implicit networks, but the two we use mainly in this work are 1) \emph{forward-backward} splitting, given by the update
\begin{equation}
\label{eqn:forward-backward}
x^{k+1} := R_G(x^k - \alpha F(x^k));
\end{equation}
and 2) \emph{Peaceman-Rachford} splitting, which is given by the iteration
\begin{equation}
\label{eqn:peaceman-rachford}
u^{k+1} = C_F C_G (u^k), \;\; x^k = R_G(u^k).
\end{equation}
Both methods will converge linearly to an $x$ that is a zero of the operator sum under certain conditions: a sufficient condition for forward-backward to converge is that $F$ be strongly monotone with parameter $m$ and Lipschitz with  constant $L$ and $\alpha < 2m/L^2$; for Peaceman-Rachford, the method will converge for any choice of $\alpha$ for strongly monotone $F$, though the convergence speed will often vary substantially based upon $\alpha$.

\section{Proofs}
\label{apx:proofs}
\renewcommand{\thetable}{\Alph{section}\arabic{table}}
\setcounter{table}{0}
\renewcommand{\thefigure}{\Alph{section}\arabic{figure}}
\setcounter{figure}{0}
\renewcommand{\thealgorithm}{\Alph{section}\arabic{algorithm}}
\setcounter{algorithm}{0}
\renewcommand{\theequation}{\Alph{section}\arabic{equation}}
\setcounter{equation}{0}

\subsection{Proof of Theorem \ref{thm:monotone-formulation}}
\begin{proof}
The proof here is immediate: the forward-backward algorithm applied to the above operators with $\alpha = 1$ corresponds exactly to the network's fixed-point iteration:
\begin{equation}
\begin{split}
z^{k+1} & = R_G(z^k - \alpha F(z^k)) \\
& = \prox_f^\alpha (z^k - \alpha (I-W)z^k + \alpha(Ux + b)) \\
& = \prox_f^1(W z^k + Ux + b).\qedhere
\end{split}
\end{equation}
\end{proof}

\subsection{Proof of Proposition \ref{prop:semidef}}
\begin{proof}
First assume $W$ is of this form.  Then clearly
\begin{equation}
(I - W)/2 + (I-W)^T/2= mI + A^T A \succeq mI.
\end{equation}
Alternatively, if $I - W \succeq mI \Longleftrightarrow (1-m)I \succeq (W+W^T)/2$, then
\begin{equation}
(W+W^T)/2 = (1-m)I - A^T A.
\end{equation}
Thus
\begin{equation}
\begin{split}
W & = (W + W^T)/2 + (W - W^T)/2 \\ & = (1-m)I - A^T A + B - B^T. \qedhere
\end{split}
\end{equation}
\end{proof}

\subsection{Proof of Theorem \ref{thm:backward}}
\begin{proof}
Differentiating both sides of the fixed-point equation $z^\star = \sigma(Wz^\star + Ux + b)$ we have
\begin{equation}
\begin{split}
\frac{\partial z^\star}{\partial (\cdot)} & = \frac{\partial \prox_f^1 (Wz^\star + U x + b)}{\partial (\cdot)} \\
& = J \left ( W \frac{\partial z^\star}{\partial (\cdot)} + \frac{\partial (Wz^\star + U x + b)}{\partial (\cdot)} \right )
\end{split}
\end{equation}
for $J$ defined in \eqref{eqn:J} (we require the Clarke generalized Jacobian owing to the fact that the nonlinearity need not be a smooth function).  Rearranging we get
\begin{equation}
\begin{split}
(I - JW) & \frac{\partial z^\star}{\partial (\cdot)} = J \frac{\partial (Wz^\star + U x + b)}{\partial (\cdot)} \\
\Leftrightarrow & \frac{\partial z^\star}{\partial (\cdot)} = (I - JW)^{-1} J \frac{\partial (Wz^\star + U x + b)}{\partial (\cdot)}.
\end{split}
\end{equation}
To show that this derivative always exists, we need to show that the $I - JW$ matrix is nonsingular.  Owing to the fact that proximal operators are monotone and non-expansive, we have $0 \leq J_{ii} \leq 1$.  First, letting $\lambda(\cdot)$ denote the set of eigenvalues of a matrix, note that
\begin{equation}
\lambda(I - JW) = \lambda(I - J^{1/2}WJ^{1/2}).
\end{equation}
This follows from the similarity transform $\lambda(I - JW) = \lambda(J^{-1/2}(I - JW)J^{1/2})$ for $J > 0$ and the case of $J_{ii} = 0$ follows via the continuity of eigenvalues taking $\lim J_{ii}\rightarrow 0$.  Now, using the fact that $0\preceq J \preceq I$, we have
\begin{equation}
\begin{split}
& \mathrm{Re}\; \lambda(I - J^{1/2}WJ^{1/2}) \\ & = \mathrm{Re}\; \lambda(I - J + J^{1/2}(I - W)J^{1/2}) > 0
\end{split}
\end{equation}
since $I - W \succeq mI$ and $I - J \succeq 0$.
\end{proof}

\subsection{Proof of Theorem \ref{thm:backward-monotone}}
\begin{proof}
We begin with the case where $J_{ii} \neq 0$ and thus $D_{ii} < \infty$.  As above, because proximal operators are themselves monotone non-expansive operators, we always $0 \leq J_{ii} \leq 1$, so that $D_{ii} \geq 0$.  Now, first assuming that $J_{ii} > 0$, and hence $D_{ii} < \infty$, we have
\begin{equation}
\begin{split}
& u = (I - JW)^{-T} v \\
\Leftrightarrow \; & (I - W^T (I+D)^{-1})u = v \\
\Leftrightarrow \; & W^{-T}u - (I+D)^{-1}u = W^{-T}v \\
\Leftrightarrow \; & (I+D) W^{-T}u - u = (I+D)W^{-T}v \\
\Leftrightarrow \; & W^{-T}u - u + D W^{-T} u = (I+D)W^{-T}v \\
\Leftrightarrow \; & \tilde{u} - W^T \tilde{u} + D \tilde{u} = (I+D)W^{-T}v
\end{split}
\end{equation}
where we define $\tilde{u} = W^{-T} u$.  To simplify the right hand side of this equation and remove the explicit $W^{-T}v$ terms\footnote{Although we could solve this operator splitting problem directly, the presence of the $W^{-T}v$ term has two notable downsides: 1) even if the $W$ matrix itself is nonsingular, it may be arbitrarily close to a singular matrix, thus making direct solutions with this matrix introduce substantial numerical errors; and 2) for operator splitting methods that do \emph{not} require an inverse of $W$ (e.g. forward-backward splitting), it would be undesirable to require an explicit inverse in the backward pass.} we note that
\begin{equation}
(I - JW)^{-T} = (I - W^TJ)^{-1} = I + (I - W^TJ)^{-1}W^T J.
\end{equation}
Thus, we can always solve the above equation with the $v$ term of the form $W^T J v$, giving
\begin{equation}
(I+D) W^{-T} W^T J v = (I+D)Jv = v.
\end{equation}
This gives us a (linear) operator splitting problem with the $\tilde{F}$ and $\tilde{G}$ operators given in \eqref{eqn:splitting_bw}.

To handle the case where $J_{ii} = 0 \Leftrightarrow D_{ii} = \infty$, we can simply take the limit $D_{ii} \rightarrow \infty$, and note that all the operators are well-defined for this case.  For instance, the resolvent operator
\begin{equation}
R_{\tilde{G}}(u) = (I + \alpha (I+D))^{-1}(u + \alpha v)
\end{equation}
and thus
\begin{equation}
R_{\tilde{G}}(u)_{ii} = \frac{u + \alpha v}{1 + \alpha(1+D_{ii})} \rightarrow 0
\end{equation}
as $D_{ii} \rightarrow \infty$.

Finally, owing to the fact that $I - W^T \succeq mI$ and $D_{ii} \geq 0$, the $\tilde{F}$ and $\tilde{G}$ operators are strongly monotone and monotone respectively, we conclude that operator splitting techniques applied to the problem will be guaranteed to converge.
\end{proof}

\section{Convolutional \mon{}s}
\label{apx:conv-mons}
\renewcommand{\thetable}{\Alph{section}\arabic{table}}
\setcounter{table}{0}
\renewcommand{\thefigure}{\Alph{section}\arabic{figure}}
\setcounter{figure}{0}
\renewcommand{\thealgorithm}{\Alph{section}\arabic{algorithm}}
\setcounter{algorithm}{0}
\renewcommand{\theequation}{\Alph{section}\arabic{equation}}
\setcounter{equation}{0}

\subsection{Inversion via the discrete Fourier transform}
\label{apx:FFT-inverse}

First consider the case where $W \in \mathbb{R}^{s^2 \times s^2}$ is the matrix representation of an unstrided (circular) convolution with a single input channel and single output channel. The convolution operates on vectorized $s\times s$ inputs. It is well known that $W$ is diagonalized by the 2D DFT operator $\mathscr{F}_s = F_s \otimes F_s$ where $F_s$ is the Fourier basis matrix $(F_s)_{ij} = \frac{1}{s}\omega^{(i-1)(j-1)}$ and $\omega = \exp(2\pi \iota / s)$. We denote $\iota = \sqrt{-1}$ to avoid confusion with the index $i$. So 
\begin{equation}
\mathscr{F}_s W\mathscr{F}_s^* = D,    
\end{equation}
a complex diagonal matrix.

Now take the case where $W \in \mathbb{R}^{ns^2 \times ns^2}$ has $n$ input and output channels. Then
\begin{equation}
\label{eqn:D}
(I_n \otimes \mathscr{F}_s)W(I_n \otimes \mathscr{F}_s^*)
= D = \left [ \begin{array}{cccc}
D_{11} & D_{12} & \cdots & D_{1n}  \\
D_{21} & D_{22} & \cdots & D_{2n} \\
\vdots & \vdots & \ddots & \vdots \\
D_{n1} & D_{n2} & \cdots & D_{nn} \end{array} \right ]
\end{equation}
where $I_n$ is the $n \times n$ identity matrix and each block $D_{ij} \in \mathbb{C}^{s^2 \times s^2}$ is a complex diagonal matrix. We will denote $\mathscr{F}_{s,n} = I_n \otimes \mathscr{F}_s$.

It is more efficient to consider the permuted form of $D$
\begin{equation}
\label{eqn:Dhat}
\hat{D} = \left [ \begin{array}{cccc}
\hat{D}^1 & 0 & \cdots & 0  \\
0 & \hat{D}^2 & \cdots & 0 \\
\vdots & \vdots & \ddots & \vdots \\
0 & 0 & \cdots & \hat{D}^{s^2}\end{array} \right ]
\end{equation}
where each block $\hat{D}^k \in \mathbb{C}^{n \times n}$, consists of the $k$th diagonal elements of all the $D_{ij}$, that is $\hat{D}^k_{ij} = (D_{ij})_{kk}$. Then inverting or multiplying by $\hat{D}$ reduces to inverting or multiplying by the diagonal blocks, which is amenable to accelerated batch-wise computation in the form of an $s^2 \times n \times n$ tensor. However, the original form \eqref{eqn:D} is more convenient mathematically and we use that here.

To perform the required inversion of the operator
\begin{equation}
    I+\alpha(I-W) = (1+\alpha m)I + \alpha A^T A - \alpha B + \alpha B^T
\end{equation}
we use the fact that $\mathscr{F}_{s,n}$ is unitary and obtain
\begin{equation}
\begin{split}
(1+\alpha m)I &+ \alpha A^T A - \alpha B + \alpha B^T \\
 &= (1+\alpha m)\mathscr{F}_{s,n}^*\mathscr{F}_{s,n} + \mathscr{F}_{s,n}^* (\alpha D_A^* \mathscr{F}_{s,n} \mathscr{F}_{s,n}^* D_A - D_B + D_B^*) \mathscr{F}_{s,n}\\
 &= \mathscr{F}_{s,n}^* ((1+\alpha m)I + \alpha D_A^* D_A - D_B + D_B^*) \mathscr{F}_{s,n}.
\end{split}
\end{equation}
The inner term here itself has the blockwise-diagonal form \eqref{eqn:D}. Thus, we can multiply a set of hidden units $z$ by the inverse of this matrix by considering the permuted form \eqref{eqn:Dhat}, inverting each block $\hat{D}^i$, taking the FFT of $z$, multiplying each corresponding block of $\mathscr{F}_{s,n}z$ by the corresponding inverse, then taking the inverse FFT.

\subsection{Zero padding}
\label{apx:zero-padding}
One drawback to the above method is that using the FFT in this manner requires that all convolutions be circular.  While empirically there is little drawback to simply replacing traditional convolutions with their circular variants, in some cases it may be desirable to avoid this setting, where information about the image may wrap around the borders.  If it is desirable to avoid this, we explicitly remove any circular dependence by zero-padding the hidden units with $(k-1)/2$ border pixels, where $k$ denotes the receptive field size of the convolution.  This zero padding can then be enforced by simply setting all the border entries to zero within the \emph{nonlinearity} of the network; because setting an element to zero is equivalent to the proximal operator for the indicator of the zero set, such operations still fit within the monotone operator setting.

\section{Multi-tier \mon{}s}
\label{apx:multi-tier-inverse}
\renewcommand{\theproposition}{\Alph{section}\arabic{proposition}}
\setcounter{proposition}{0}
\renewcommand{\thetable}{\Alph{section}\arabic{table}}
\setcounter{table}{0}
\renewcommand{\thefigure}{\Alph{section}\arabic{figure}}
\setcounter{figure}{0}
\renewcommand{\thealgorithm}{\Alph{section}\arabic{algorithm}}
\setcounter{algorithm}{0}
\renewcommand{\theequation}{\Alph{section}\arabic{equation}}
\setcounter{equation}{0}

\subsection{Parameterization}
Recall the setting of Section \ref{subsec:multi-tier}, with

\begin{equation}
z = \left [ \begin{array}{c} z_1 \in \mathbb{R}^{n_1 s_1^2} \\
z_2 \in \mathbb{R}^{n_2 s_2^2} \\
\vdots \\
z_L \in \mathbb{R}^{n_L s_L^2}
\end{array} \right ],
~~~~~~~W = \left [ \begin{array}{ccccc}
W_{11} & 0 & 0 & \cdots & 0  \\
W_{21} & W_{22} & 0 & \cdots & 0 \\
0 & W_{32} & W_{33} & \cdots & 0  \\
\vdots & \vdots & \vdots & \ddots & \vdots \\
0 & 0 & 0 & \cdots & W_{LL} \end{array} \right ].
\label{eqn:multi-tier_form}
\end{equation}

To ensure $W$ has the form $(1-m)I - A^T A + B - B^T$, we restrict both $A$ and $B$ to have the same bidiagonal structure as $W$. Then the diagonal terms $W_{ii}$ have the form
\begin{equation}
W_{ii} = (1-m)I - A_{ii}^T A_{ii} - A_{i+1,i}^T A_{i+1,i} + B_{ii} - B_{ii}^T
\end{equation}
for $i<L$ and
\begin{equation}
W_{LL} = (1-m)I - A_{LL}^T A_{LL} + B_{LL} - B_{LL}^T.
\end{equation}

To compute the off-diagonal terms $W_{i+1,i}$ note that restricting $W$ to be bidiagonal makes the off-diagonal terms of $B$ redundant. E.g. since $W_{12} = 0$, then
\begin{equation}
\begin{split}
& -A_{21}^T A_{22} - B_{21}^T = W_{12} = 0\\
\Rightarrow \; & W_{21} = -A_{22}^T A_{21} + B_{21} = -2 A_{22}^T A_{21}.\\
\end{split}
\end{equation}

\subsection{Inversion via the discrete Fourier transform} Consider $W$ of the form (\ref{eqn:multi-tier_form}) with convolutions
\begin{equation}
\label{eqn:multi-tier_form_2}
\begin{split}
W_{ii} &= (1-m)I - A_{ii}^T A_{ii} - A_{i+1,i}^T A_{i+1,i} + B_{ii} - B_{ii}^T\\
W_{LL} &= (1-m)I - A_{LL}^T A_{LL} + B_{LL} - B_{LL}^T\\
W_{i+1,i} &= -2 A_{i+1,i+1}^T A_{i+1,i}.
\end{split}
\end{equation}
Here the $A_{ii}$ and $B_{ii}$ terms are unstrided convolutions with $n_i$ input and $n_i$ output channels, while the $A_{i,i+1}$ are strided convolutions with $n_i$ input channels and $n_{i+1}$ output channels.

In order to multiply by $(I+\alpha(I-W))^{-1}$, we use back substitution to solve for $x$ in  
\begin{equation}
    z = (I+\alpha(I-W)) x.
\end{equation} Let $W' = (I+\alpha(I-W))$. The back substitution proceeds by tiers, i.e.
\begin{equation}
\begin{split}
x_1 &= W'^{-1}_{11} z_1\\
x_2 &= W'^{-1}_{22} (z_2 - W'_{21}x_1)\\
x_3 &= W'^{-1}_{33} (z_3 - W'_{32}x_2)\\
\vdots
\end{split}
\end{equation}
Therefore only the diagonal blocks $W'_{ii}$ need be inverted. The inversion of e.g.
\begin{equation}
W'_{11} = (1 + \alpha  m)I + \alpha (A_{11}^T A_{11} +  A_{21}^T A_{21} + B_{11} - B_{11}^T)\end{equation}
is complicated by the fact that $A_{21}$ is strided, so that it is no longer diagonalized by the DFT. Instead, we perform inversion using the following proposition.

\begin{proposition}
Let $A \in\mathbb{R}^{n_1s^2 \times n_1s^2}$ be an unstrided circular convolution with $n_1$ input and $n_1$ output channels, and $B \in\mathbb{R}^{n_2s^2 \times n_1s^2}$ a strided circular convolution with $n_1$ input and $n_2$ output channels and stride $r$ where $r$ divides $s$. Then
\begin{equation}
(A + B^T B)^{-1} = \mathscr{F}_{s,n_1}^*(D^{-1}_A - D^{-1}_A D^*_B (I_{n_2} \otimes K) D_B D^{-1}_A) \mathscr{F}_{s,n_1}
\end{equation}
where
\begin{equation}
\begin{split}
D_A &= \mathscr{F}_{s,n_1}A \mathscr{F}_{s,n_1}^*, \;\;\;\; D_B = \mathscr{F}_{s,n_2}B \mathscr{F}_{s,n_1}^*,\\
K &= S^T J (s^2r^2I + J^T S D_B D^{-1}_A D^*_B S^T J)^{-1} J^T S\\
\end{split}
\end{equation}
where $J = 1_{r^2} \otimes I_{s^2/r^2}$ is $r^2$ stacked identity matrices of size $(s^2/r^2) \times (s^2/r^2)$ and \\$S = (I_r \otimes S_{s/r,s})$ is a permutation matrix where $S_{a,b} \in \mathbb{R}^{ab \times ab}$ denotes the perfect shuffle matrix defined by subselecting rows of the identity matrix $I_{ab}$, here given in \textsc{MATLAB} notation:
\begin{equation}
S_{a,b} = \left [ \begin{array}{c}
I_{ab}(1:b:ab, :)  \\
I_{ab}(2:b:ab, :)  \\
\vdots \\
I_{ab}(b:b:ab, :) \end{array} \right ].
\end{equation}
\end{proposition}

\begin{proof} We will show that
\begin{equation}
A + B^T B = \mathscr{F}_{s,n_1}^*(D_A + D^*_B (I_{n_2} \otimes (\frac{1}{s^2r^2}S^T J J^T S)) D_B)\mathscr{F}_{s,n_1}.
\end{equation}
The desired result then follows by applying the Woodbury matrix idenetity.

We start by breaking $B$ into an unstrided convolution $B'$ which can be diagonalized by the DFT and a matrix $U_{r,s}$ which performs the striding on each channel:
\begin{equation}
B = (I_{n_2} \otimes U_{r,s})B' = (I_{n_2} \otimes U_{r,s})\mathscr{F}_{s,n_2}^*D_B\mathscr{F}_{s,n_1}
\end{equation}
where $U_{r,s}\in \mathbb{R}^{(s^2/r^2) \times s^2}$ is defined by subselecting rows of the identity matrix:
\begin{equation}
U_{r,s} = \left [ \begin{array}{c}
I_{s^2}(1:r:s, :)  \\
I_{s^2}(rs + 1:r:(r+1)s, :)  \\
I_{s^2}(2rs + 1:r:(2r+1)s, :)  \\
\vdots \\
I_{s^2}(s^2 - sr + 1:r:s^2 - s(r-1), :) \end{array} \right ].
\end{equation}
So
\begin{equation}
B^T B = \mathscr{F}_{s,n_1}^*D^*_B(I_{n_2} \otimes (\mathscr{F}_s U_{r,s}^T U_{r,s}\mathscr{F}^*_s))D_B\mathscr{F}_{s,n_1}.
\end{equation}
We want to show that $\mathscr{F}_s U_{r,s}^T U_{r,s}\mathscr{F}_s^* = \frac{1}{s^2r^2}S^TJJ^TS$. Observe that
\begin{equation}
U_{r,s}^T U_{r,s} = (T_{r,s} \otimes T_{r,s})
\end{equation}
where $T_{r,s} \in \mathbb{R}^{s \times s}$ is given by
\begin{equation} (T_{r,s})_{ij} = 
    \begin{cases}
        1  & \mbox{if } i=j\mbox{ and }i \mbox{~(mod } r) = 1,\\
        0  & \mbox{else.}
    \end{cases}
\end{equation}
Then by the properties of Kronecker product
\begin{equation}
\mathscr{F}_s U_{r,s}^T U_{r,s}\mathscr{F}_s^* = (F_s \otimes F_s) (T_{r,s} \otimes T_{r,s})(F_s^* \otimes F_s^*) = (F_s T_{r,s} F_s^*)\otimes (F_s T_{r,s} F_s^*).
\end{equation}

We now show that $(F_s T_{r,s} F_s^*) = L$ where
\begin{equation} L_{ij} =  
    \begin{cases}
        \frac{1}{sr}  & \mbox{if } i \equiv j~(\text{mod } s/r),\\
        0  & \mbox{else.}
    \end{cases}
\end{equation}
To do so we use several properties of the roots of unity $z^k = \exp(2\pi\iota k/s)$.
\begin{enumerate}
    \item If $a \equiv b~(\text{mod } s)$ then $z^a = z^b$.
    \item If $z$ is a primitive $s$th root of unity then $z^m$ is a primitive $a$th root of unity where $a = \frac{s}{\text{gcd}(m,s)}$.
    \item The sum of the $s$th roots of unity $\sum_{k=0}^{s-1} z^k = 0$ if $s>1$.
\end{enumerate}
We first compute $L_{ij}$ for the case when $i\equiv j~(\text{mod } s/r)$, or in other words $i = j + \frac{ks}{r}$ for some integer $k$. We have
\begin{equation}
\begin{split}
L_{ij} &= \frac{1}{s^2} \sum_{p=1:r:s} \omega^{(i-1)(p-1)}\bar{\omega}^{(p-1)(j-1)}\\
&=\frac{1}{s^2} \sum_{p=0:r:s-1} \exp(2\pi\iota p(i-j)/s)\\
&=\frac{1}{s^2} \sum_{p=0:r:s-1} \exp(2\pi\iota pk/r)\\
&=\frac{1}{s^2} \sum_{p=0}^{\frac{s}{r}-1} \exp(2\pi\iota pk)\\
&= \frac{1}{sr}.
\end{split}
\end{equation}
For the case when $i\not\equiv j~(\text{mod } s/r)$, or in other words $i = j + \frac{ks}{r} + m$ for some integers $k$ and $m$ with $-\frac{s}{r} < m < \frac{s}{r}$, we have
\begin{equation}
\begin{split}
L_{ij} &= \frac{1}{s^2} \sum_{p=1:r:s} \omega^{(i-1)(p-1)}\bar{\omega}^{(p-1)(j-1)}\\
&=\frac{1}{s^2} \sum_{p=0:r:s-1} \exp(2\pi\iota p(i-j)/s)\\
&=\frac{1}{s^2} \sum_{p=0}^{\frac{s}{r}-1} \exp(2\pi\iota p(i-j)r/s)\\
&=\frac{1}{s^2} \sum_{p=0}^{\frac{s}{r}-1} \exp(2\pi\iota pmr/s)\exp(2\pi\iota pk).\\
\end{split}
\end{equation}
By property (2), since $\exp(2\pi\iota r/s)$ is a primitive $\frac{s}{r}$th root of unity, then $\exp(2\pi\iota mr/s)$ is a primitive $d$th root of unity where $d=\frac{s/r}{\text{gcd}(m,s/r)}$. Since $d$ divides $s/r$, we can split the sum into several sums of $d$th roots of unity using property (1), each of which will sum to zero by property (3).
\begin{equation}
\begin{split}
L_{ij}  &=\frac{1}{s^2} \sum_{p=0}^{\frac{s}{r}-1} \exp(2\pi\iota pmr/s)\\
&= \frac{1}{s^2} \sum_{q=0}^{\frac{s}{rd}-1}\sum_{p=0}^{ d-1} \exp(2\pi\iota (p+qd)mr/s)\\
&= \frac{1}{srd} \sum_{p=0}^{ d-1} \exp(2\pi\iota pmr/s)\\
&= 0\\
\end{split}
\end{equation}
where the second equality follows from property (1) since $p = p + qd~(\text{mod } d)$ and each sum in the third line is zero by property (3) since $\exp(2\pi\iota mr/s)$ is a primitive $d$th root of unity.

We now have $\mathscr{F}_s U_{r,s}^T U_{r,s}\mathscr{F}_s^* = L \otimes L$ and it remains to use properties of Kronecker product to show that $L \otimes L = \frac{1}{s^2r^2}S^TJJ^TS$. In particular we need associativity and the fact that for $A\in \mathbb{R}^{n\times n}$ and $B\in \mathbb{R}^{m\times m}$, we have
\begin{equation}
B \otimes A = S_{n,m}(A \otimes B)S_{n,m}^T
\end{equation}
where $S_{n,m}$ is the perfect shuffle matrix.
Note that $L = \frac{1}{sr}1_{r \times r} \otimes I_{s/r}$ where $1_{r \times r}$ is the $r \times r$ matrix of all ones. Then
\begin{equation}
\begin{split}
L \otimes L &= \frac{1}{s^2r^2} (1_{r \times r} \otimes I_{s/r}) \otimes (1_{r \times r} \otimes I_{s/r})\\
&= \frac{1}{s^2r^2} 1_{r \times r} \otimes (I_{s/r} \otimes (1_{r \times r} \otimes I_{s/r}))\\
&= \frac{1}{s^2r^2} 1_{r \times r} \otimes (S_{s/r,s}  ((1_{r \times r} \otimes I_{s/r})\otimes I_{s/r}) S_{s/r,s}^T)\\
&= \frac{1}{s^2r^2} 1_{r \times r} \otimes (S_{s/r,s}  (1_{r \times r} \otimes I_{s^2/r^2}) S_{s/r,s}^T)\\
&= \frac{1}{s^2r^2} (I_r \otimes S_{s/r,s}) (1_{r^2 \times r^2} \otimes I_{s^2/r^2})(I_r \otimes S_{s/r,s}^T)\\
&= \frac{1}{s^2r^2} S JJ^T S^T\\
\end{split}
\end{equation}
which completes the proof.
\end{proof}

\section{Experiment details}
\label{apx:experiment_details}
\renewcommand{\thetable}{\Alph{section}\arabic{table}}
\setcounter{table}{0}
\renewcommand{\thefigure}{\Alph{section}\arabic{figure}}
\setcounter{figure}{0}
\renewcommand{\thealgorithm}{\Alph{section}\arabic{algorithm}}
\setcounter{algorithm}{0}
\renewcommand{\theequation}{\Alph{section}\arabic{equation}}
\setcounter{equation}{0}

\begin{table*}[t]
	\centering
	\scriptsize
	\begin{tabular}{l r r r r}
		\toprule
		     & \multicolumn{4}{c}{\textbf{CIFAR-10}} \\\midrule
	            & Single conv & Multi-tier & Single conv lg. & Multi-tier lg. \\ \midrule
		Num. channels & 81 & (16,32,60) & 200 & (64,128,128) \\
		Num. params & 172,218 & 170,194  & 853,612 & 1,014,546 \\
		Epochs & 40 & 40 & 65 & 65 \\
		Initial lr & 0.001 & 0.01  & 0.001 & 0.001  \\
		Lr schedule & step decay & step decay & 1-cycle & 1-cycle \\
		Lr decay steps & 25 & 10  & - & - \\
		Lr decay factor & 10 & 10  & - & - \\
		Max learning rate & - & -  & 0.01 & 0.05 \\
		Data augmentation & - & - & \checkmark & \checkmark\\
		\midrule
		&&&\\
	\end{tabular}

	\begin{tabular}{l r r | r r r}
	    \toprule
		     &  \multicolumn{2}{c}{\textbf{SVHN}} & \multicolumn{3}{c}{\textbf{MNIST}} \\ \midrule
	         & Single conv & Multi-tier & FC & Single conv & Multi-tier \\ \midrule
	    Num. channels & 81 & (16,32,60) & 87* & 54 & (16, 32, 32) \\
		Num. params & 172,218 & 170,194  & 84,313 & 84,460 & 81,394 \\
		Initial lr & 0.001 & 0.001  & 0.001 & 0.001 & 0.001 \\
		Epochs & 40 & 40  & 40 & 40 & 40 \\
		Lr decay steps & 25 & 10 & 10 & 10 & 10\\
		Lr decay factor \% & 10 & 10 & 10 & 10 & 10\\
	    \bottomrule

	\end{tabular}
	\caption{Model hyperparameters. *FC is a dense layer with output dimension of 87.}
	\label{tab:params}
\end{table*}

\subsection{Model architecture}
Recall that a \mon{} is defined by a choice of linear operators $W$ and $U$, bias $b$, and nonlinearity $\sigma$, and that we parameterize $W$ via linear operators $A$ and $B$. For all experiments we use $\sigma=\text{ReLU}$. In the fully-connected network $A,B$ and $U$ are dense matrices; in the single-convolution network they are unstrided convolutions with kernel size 3. The structure of the multi-tier network is as described in (\ref{eqn:multi-tier_form}) and (\ref{eqn:multi-tier_form_2}); we use three tiers with unstrided convolutions for $U$ and $A_{ii},B_{ii}$ and stride-2 convolutions for the subdiagonal terms $A_{i,i+1}$, all with kernels of size 3. The number of channels for single and multi-tier convolutional models varies by dataset, as shown in Table \ref{tab:params}.

For all models, the fixed point $z^\star$ is mapped to logits $\hat{y}$ via a dense output layer, and the single convolution model first applies 4$\times$4 average pooling:
    $$\hat{y} = W_o z^\star + b_o~~~~~~\text{or}~~~~~~\hat{y} = W_o~\text{AvgPool}_{4\times4}(z^\star) + b_o.$$

\subsection{Training details}
 Because $W = (1-m)I - A^T A + B - B^T$ contains both linear and quadratic terms, we find that a variant of weight normalization helps to keep the gradients of the different parameters on the same scale. For example, when $W$ is a dense matrix, we reparameterize $A^TA$ as $g\frac{A^TA}{\|A\|^2}$ and $B$ as $h\frac{B}{\|B\|}$, where $g$ and $h$ are learned scalars. When $W$ consists of a single or multi-tiered convolutions, we reparameterize each convolution kernel analogously.

 All models are trained by running Peaceman-Rachford with error tolerence $\epsilon=$1e-2, which reduces the number of iterations without impacting performance. The monotonicity parameter $m$ also affects convergence speed since it controls the contraction factor of the relevant operators; consistent with this, we find that Peaceman-Rachford takes longer to converge for smaller $m$, and use $m=1$ for all models since model performance is not sensitive to $m\in [0.01,1]$. We also find that the Lipschitz parameter $L$ of $I-W$ increases during training, changing the optimal $\alpha$ value. We therefore tune $\alpha\in\{1,1/2,1/4,\ldots\}$ over the course of training so as to minimize forward-pass iterations.
 
 One detail about stopping criteria for the splitting method: computing the residual \linebreak $\|z^{k+1} -  f(z^{k+1})\|/\|z^{k+1}\|$ requires an additional call to the function $f$. Therefore during training we instead use the criterion $\|z^{k+1} -  z^{k}\|/\|z^{k+1}\| \leq \epsilon$. The error shown in Figure \ref{fig:convergence} is the former, while the stopping criterion used in Figures \ref{fig:tuning_alpha_fw} and \ref{fig:tuning_alpha_bw} is the latter. Technically this latter criterion itself depends on both $\alpha$ and $L$; for different $\alpha$ and $L$ values, having $\|z^{k+1} -  z^{k}\|/\|z^{k+1}\| \leq \epsilon$ implies different bounds on the residual. However, we find that this effect is minimal, so that both stopping criteria work equally well in practice.
 
Table \ref{tab:params} gives details of the training hyperparameters used for each model. All models are trained with ADAM \citep{DBLP:journals/corr/KingmaB14}, using batch size of 128. For all but the large CIFAR-10 models, the initial learning rate is chosen from \{1e-2, 1e-3\} and decayed by a factor of 10 after every 10 or 25 epochs, and the default ADAM momentum parameters are used. All training data is normalized to mean $\mu=0$, standard deviation $\sigma=1$.

\paragraph{CIFAR-10 with data augmentation} When training large models on CIFAR-10 we use standard data augmentation, consisting of zero-padding the 32$\times$32 images to 40$\times$40, then randomly cropping back to 32x32, and finally performing random horizontal flips. In order to reduce the number of training epochs, we use a single cycle of increasing and decreasing learning rate to achieve super-convergence \cite{smith2019super}. The learning rate is increased from 1e-3 to the max learning rate (see Table \ref{tab:params}) over 30 epochs, then decreased back to 1e-3 over 30 epochs, then held at 1e-3 for 5 epochs. (The max learning rate is chosen by training for a single epoch while increasing the learning rate until the loss diverges.) The momentum is also decreased from 0.95 to 0.85 over 30 epochs, then back to 0.95 over 30 epochs, then held at 0.95 for 5 epochs. However, we note that the model obtains the same performance when trained with constant learning rate of 1e-3 for around 200 epochs.

\subsection{Dataset statistics}

MNIST \cite{lecun2010mnist} consists of black and white examples of handwritten digits 0-9. SVHN \cite{netzer2011reading} consists of color images of digits 0-9 extracted from house numbers captured by Google Stree View. CIFAR-10 \cite{krizhevsky2009learning} consists of small images from 10 object classes. Dataset statistics are shown in Table \ref{tab:data}.

\begin{table*}[t]
	\centering
	\scriptsize
	\begin{tabular}{l r r r r}
		\toprule
		     & \textbf{Train examples} & \textbf{Test examples} & \textbf{Image dim.} & \textbf{Num. channels} \\ \midrule
		     MNIST & 60,000 & 10,000 & 28 $\times$ 28 & 1\\
		     SVHN  & 73,257 & 26,032 & 32 $\times$ 32 & 3\\
		     CIFAR-10 & 50,000 & 10,000 & 32 $\times$ 32 & 3\\
	    \bottomrule

	\end{tabular}
	\caption{Dataset statistics}
	\label{tab:data}
\end{table*}

\newpage
\section{Additional results and figures}
\label{apx:figs}
\renewcommand{\thetable}{\Alph{section}\arabic{table}}
\setcounter{table}{0}
\renewcommand{\thefigure}{\Alph{section}\arabic{figure}}
\setcounter{figure}{0}
\renewcommand{\thealgorithm}{\Alph{section}\arabic{algorithm}}
\setcounter{algorithm}{0}
\renewcommand{\theequation}{\Alph{section}\arabic{equation}}
\setcounter{equation}{0}


\begin{figure}[h]
\centering
\begin{subfigure}{.5\textwidth}
  \centering
  \includegraphics{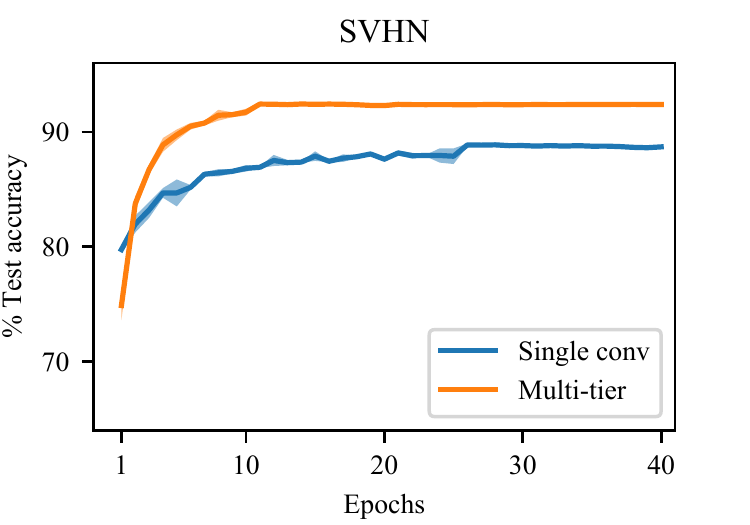}
\end{subfigure}%
\begin{subfigure}{.5\textwidth}
  \centering
  \includegraphics{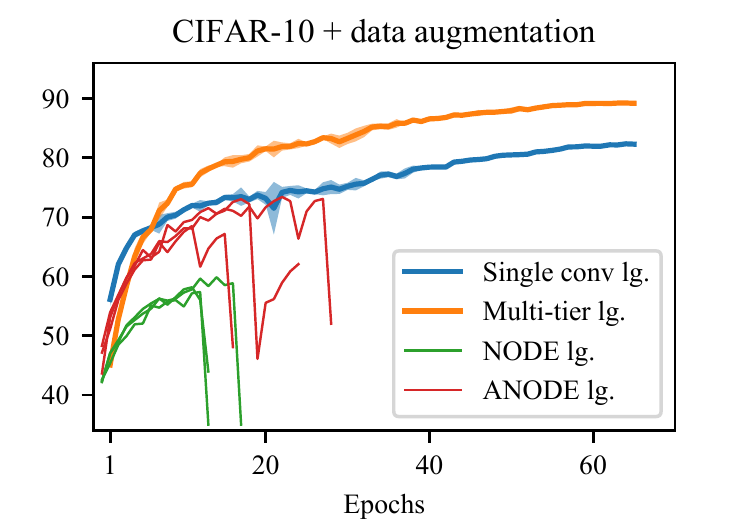}
\end{subfigure}\\
\begin{subfigure}{.5\textwidth}
  \centering
    \includegraphics{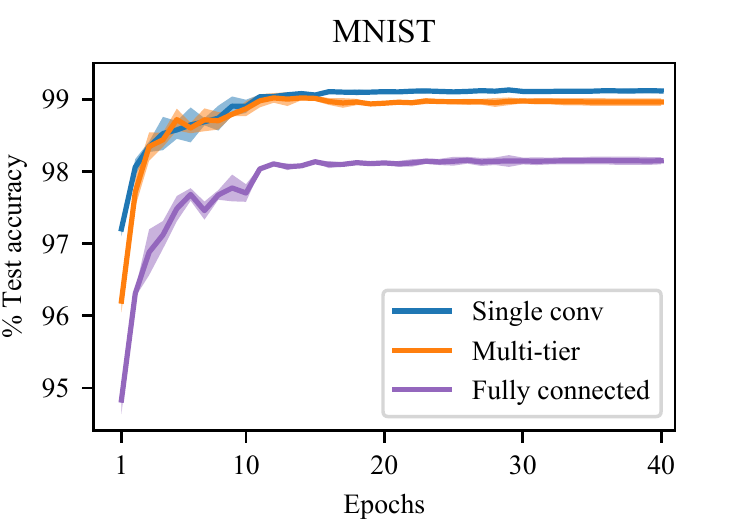}
\end{subfigure}
\caption{Test accuracy of \mon{}s and Neural ODE models during training.}
\label{fig:training}
\end{figure}

\begin{figure}[h]
\centering
  \begin{subfigure}{.5\textwidth}
    \centering
    \includegraphics{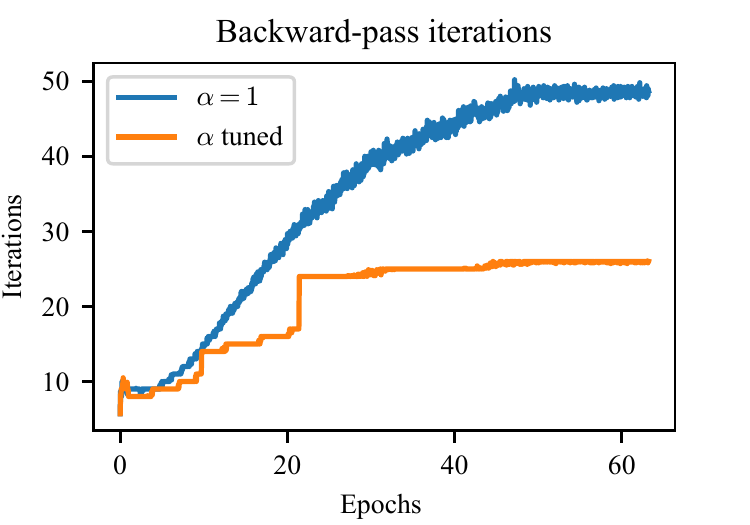}
  \end{subfigure}
\caption{Iterations required by Peaceman-Rachford backprop over the course of training.}
\label{fig:tuning_alpha_bw}
\end{figure}

\end{document}